\documentclass[nohyperref]{article}

\def\ICML{1}  
\def\tdotoggle{0}  
\def\showExperimental{0}  

\newcommand{\OnlyICML}[1]{\ifnum\ICML=1 #1 \fi}
\newcommand{\OnlyArXiv}[1]{\ifnum\ICML=0 #1 \fi}
\newcommand{\Experimental}[2]{\ifnum\showExperimental=1{\textcolor{WildStrawberry}{#1}}\else{#2}\fi}

\usepackage{microtype}
\usepackage{graphicx}
\usepackage{subfigure}
\usepackage{booktabs} 
\usepackage{gensymb}
\usepackage{tabularx}
\usepackage{xspace}
\usepackage{hyperref}

\usepackage{amsmath}
\usepackage{amssymb}
\usepackage{mathtools}
\usepackage{amsthm}

\usepackage{enumitem}
\usepackage{thmtools,thm-restate}
\usepackage[usenames,dvipsnames]{xcolor}

\OnlyICML{
	\usepackage[accepted]{icml2022}  
	
	\usepackage[capitalize,noabbrev]{cleveref}
}
\OnlyArXiv{
	\usepackage[margin=2cm]{geometry}
	\usepackage{natbib}
	\hypersetup{
		colorlinks=true,
		linkcolor=Blue,
		citecolor=black!30!ToCgreen,
		filecolor=Maroon,
		urlcolor=Maroon
	}
	\usepackage{times}
	\usepackage{inconsolata}
	\usepackage[capitalize,noabbrev,nameinlink]{cleveref}
}

\theoremstyle{plain}
\newtheorem{theorem}{Theorem}[section]

\newtheorem{lemma}[theorem]{Lemma}

\newtheorem{claim}[theorem]{Claim}
\newtheorem{conjecture}[theorem]{Conjecture}
\crefname{conjecture}{Conjecture}{Conjectures}

\crefname{observation}{Observation}{Observations}
\theoremstyle{definition}
\newtheorem{definition}[theorem]{Definition}
\newtheorem{assumption}[theorem]{Assumption}
\crefname{assumption}{Assumption}{Assumptions}
\newtheorem{property}[theorem]{Property}
\crefname{property}{Property}{Properties}
\theoremstyle{remark}

\usepackage{tikz,pgfplots}
\usetikzlibrary{pgfplots.groupplots}

\definecolor{pw}{HTML}{7977B8}
\definecolor{og}{HTML}{3C8031}
\definecolor{maroon}{HTML}{AF3235}
\definecolor{yo}{HTML}{FAA21A}
\definecolor{mybrick}{RGB}{180,14,15}
\definecolor{Gred}{RGB}{219, 50, 54}
\definecolor{Ggreen}{RGB}{60, 186, 84}
\definecolor{Gblue}{RGB}{72, 133, 237}
\definecolor{Gyellow}{RGB}{247, 178, 16}
\definecolor{ToCgreen}{RGB}{0, 128, 0}
\definecolor{myGold}{RGB}{231,141,20}
\definecolor{myBlue}{rgb}{0.19,0.41,.65}
\definecolor{myPurple}{RGB}{175,0,124}

\usepackage[colorinlistoftodos,prependcaption,textsize=scriptsize]{todonotes} 
\providecommand{\tdotoggle}{2}
\ifnum\tdotoggle>0
\ifnum\ICML=1
	\paperwidth=\dimexpr \paperwidth + 3.5cm\relax
	\oddsidemargin=\dimexpr\oddsidemargin + 2cm\relax
	\evensidemargin=\dimexpr\evensidemargin + 2cm\relax
	\marginparwidth=\dimexpr \marginparwidth + 1.6cm\relax
\else
	\paperwidth=\dimexpr \paperwidth + 2.8cm\relax
	\marginparwidth=\dimexpr \marginparwidth + 1.6cm\relax
\fi
\fi
\newcommand{\mytodo}[1]{\ifnum\tdotoggle>0{#1}\fi}

\newcommand{\unsure}[1]{\mytodo{\todo[linecolor=pw,backgroundcolor=pw!25,bordercolor=pw]{#1}}}

\newcommand{\info}[1]{\ifnum\tdotoggle>1{\todo[linecolor=og,backgroundcolor=og!25,bordercolor=og]{#1}}\fi}
\newcommand{\nishanth}[1]{\mytodo{\todo[linecolor=purple,backgroundcolor=purple!25,bordercolor=purple]{ND: #1}}}

\newcommand{\tableoftodos}{\ifnum\tdotoggle=1 \listoftodos[Comments/To Do's] \fi}

\DeclareMathOperator*{\Ex}{\mathbb{E}}

\DeclareMathOperator*{\argmin}{\mathrm{argmin}}

\newcommand{\what}[1]{\widehat{#1}}
\newcommand{\wtilde}[1]{\widetilde{#1}}

\newcommand{\set}[1]{\left \{ #1 \right \}}

\newcommand{\inparen}[1]{\left ( #1 \right )}
\newcommand{\insquare}[1]{\left [ #1 \right ]}

\newcommand{\inangle}[1]{\left \langle #1 \right \rangle}
\newcommand{\infork}[1]{\left \{ \begin{matrix} #1 \end{matrix} \right .}

\newcommand{\bbR}{\mathbb{R}}
\newcommand{\bbS}{\mathbb{S}}

\newcommand{\calC}{\mathcal{C}}
\newcommand{\calD}{\mathcal{D}}

\newcommand{\calL}{\mathcal{L}}

\newcommand{\calX}{\mathcal{X}}

\newcommand{\imgPath}{plots}  

\OnlyICML{

\icmltitlerunning{Do More Negative Samples Necessarily Hurt In Contrastive Learning?}
}

\newcommand{\NCE}{\mathrm{NCE}}
\newcommand{\DNCE}{\calD_{\NCE}}

\newcommand{\LNCEk}{\calL_{\NCE}^{(k)}}
\newcommand{\Lsup}{\calL_{\mathrm{sup}}}
\newcommand{\Dsup}{\calD_{\mathrm{sup}}}
\newcommand{\logloss}{\ell_{\mathrm{log}}}
\newcommand{\hingeloss}{\ell_{\mathrm{hinge}}}

\newcommand{\CS}{\mathsf{CS}}
\newcommand{\IntraVar}{\mathsf{Intra}\text{-}\mathsf{Var}}

\newcommand{\tuple}[1]{\inangle{#1}}

\begin{document}

\OnlyICML{
	\twocolumn[
	\icmltitle{Do More Negative Samples Necessarily Hurt In Contrastive Learning?}
	
	
	
	\icmlsetsymbol{equal}{*}
	\begin{icmlauthorlist}
	\icmlauthor{Pranjal Awasthi}{equal,google}
	\icmlauthor{Nishanth Dikkala}{equal,google}
	\icmlauthor{Pritish Kamath}{equal,google}
	\end{icmlauthorlist}
	\icmlaffiliation{google}{Google Research, USA}

	\icmlcorrespondingauthor{Nishanth Dikkala}{nishanthd@google.com}
	\icmlcorrespondingauthor{Pritish Kamath}{pritish@alum.mit.edu}
	
	\icmlkeywords{Machine Learning, ICML}
	
	\vskip 0.3in
	]

	
	
	\printAffiliationsAndNotice{\icmlEqualContribution} 
}\OnlyArXiv{
	\vspace*{-3mm}
	\begin{center}
		\hrule height 2.5pt
		\vspace{7mm}{\LARGE Do More Negative Samples Necessarily Hurt In Contrastive Learning?}\\[7mm]
		\hrule height 1.25pt
		\vspace*{4mm}
	
		\newcolumntype{x}{>{\centering\arraybackslash}p{5cm}}
		\begin{tabular}{xxx}
			Pranjal Awasthi$^*$ & Nishanth Dikkala$^*$ & Pritish Kamath$^*$ \\
			\texttt{pranjalawasthi@google.com} & \texttt{nishanthd@google.com} & \texttt{pritish@alum.mit.edu}\\[2.7mm]
			\multicolumn{3}{c}{\large $^*$Google Research}
		\end{tabular}
		\vspace{2mm}
	\end{center}
}
\begin{abstract}
Recent investigations in noise contrastive estimation suggest, both empirically as well as theoretically, that while having more ``negative samples'' in the contrastive loss improves downstream classification performance initially, beyond a threshold, it hurts downstream performance due to a ``collision-coverage'' trade-off. But is such a phenomenon inherent in contrastive learning?
We show in a simple theoretical setting, where positive pairs are generated by sampling from the underlying latent class (introduced by Saunshi et al. (ICML 2019)), that the downstream performance of the representation optimizing the (population) contrastive loss in fact does not degrade with the number of negative samples. Along the way, we give a structural characterization of the optimal representation in our framework, for noise contrastive estimation. We also provide empirical support for our theoretical results on CIFAR-10 and CIFAR-100 datasets.
\end{abstract}

\section{Introduction}
\label{sec:intro}
Unsupervised representation learning aims to extract semantically meaningful features from complex high-dimensional inputs without a supervised signal \citep{bengio2013representation}. These representations are then meant to be useful for a host of downstream supervised tasks. The benefits of successfully executing such a paradigm for learning are twofold: (1) labelled data is expensive and in contrast, unlabelled data is abundant and easy to get, (2) rather than building a specialized model for each downstream task we get to learn a general-purpose representation which makes solving each downstream task simpler which makes it possible to scalably solve multiple downstream tasks in an efficient manner.

Noise contrastive estimation (NCE) \citep{gutmann2010noise} (also known as Contrastive Learning) has emerged as a highly effective approach for unsupervised representation learning using deep networks \citep{chen2020simple, chen2020intriguing, tian2021understanding, grill2020bootstrap}. This approach tries to minimize the distance between representations of semantically similar inputs, while maximizing the distance between the representations of semantically dissimilar inputs. More concretely, a mathematical abstraction for representation learning with {\em Noise Contrastive Estimation (NCE)} for an input space $\calX$ is: (1) a single NCE example consists of $k + 2$ raw inputs, $(x, x^+, x_1^-, \ldots, x_k^-)$, where $(x, x^+)$ are ``semantically similar'' and $x_i^-$ are sampled from the same marginal distribution as $x$\info{PK: I changed ``some base distribution'' to ``the same marginal distribution as $x$''. Is that restrictive?}, (2) the representation $f : \calX\to \bbR^d$ is trained to encourage $f(x)^\top f(x^+) \gg f(x)^\top f(x_i^-)$ for each $i$. This second step can be done with standard classification objectives, such as the cross-entropy loss, where the model is viewed as a classifier over $k + 1$ labels. For example, a candidate objective to minimize (called NCE loss) is
\begin{align}
	\Ex\limits_{x, x^+, x_{1:k}^-} \log \inparen{ \frac{e^{f(x)^\top f(x^+)} + \sum_{i=1}^k e^{f(x)^\top f(x_i^-)}}{e^{f(x)^\top f(x^+)}}} \label{eq:contrastive-loss-logistic}
\end{align}
To have a sense of scale it is assumed that $\|f(x)\|_2 = 1$ for all $x$. Such a normalization is also standard in practice \citep{chen2020simple, wang2020understanding, zimmermann2021contrastive}. Following standard terminology, we will refer to $(x, x^+)$ as a positive pair and $x_i^-$ as negative examples.

Contrastive learning combined with deep neural networks has recently shown highly promising empirical results in the vision and NLP paradigms \citep{smith2005contrastive,mikolov2013efficient,schroff15facenet,chen2020simple,oord2018representation,wang2021dense,clark2020pre}. Despite this empirical success, it is not well understood why a good representation learnt in this manner works well for downstream tasks. In particular, there are many design choices present in the formulation which can affect the quality of the representations learnt. Some of the salient ones are the number of negative examples per sample $k$, the choice of the architecture, the distribution of positive pairs $(x, x^+)$, the hyper-parameters of the optimization algorithm among others. In this paper, we focus primarily on the number of negative examples $k$.

Prior empirical work observed that increasing $k$ helps improve the quality of the representations~\cite{chen2020simple}. However, in a theoretical framework proposed by
\cite{saunshi2019theoretical} to analyze the properties of NCE, it is argued that increasing $k$ beyond a certain point can degrade performance due to an increased chance of seeing negative samples which have the same latent features as $x, x^+$. These are unintended \emph{collisions} and \citet{saunshi2019theoretical} argue that too many collisions might make it harder for the model to learn good representations, evidenced by the degradation of an {\em upper bound} they show on the supervised learning loss in terms of the NCE loss.

The follow-up work of \cite{ash21investigating} also re-iterates this message by proposing that although the quality of the learnt representation initially improves with increasing $k$ due to improved \emph{coverage}, beyond a point it starts to degrade exponentially fast with $k$. Hence the work of \citet{ash21investigating} proposes that a \emph{collision-coverage trade-off} is inherent in contrastive learning, suggesting that the optimal value of $k$ should scale with the number of underlying concepts in the data.
Again, this is evidenced by the degradation of an {\em upper bound} they show on supervised learning loss in terms of the NCE loss.
However, this line of reasoning has an issue that the supervised loss is bounded, even for a fixed representation, whereas the contrastive loss can grow to $\infty$ as $k \to \infty$, even for the ``best'' representation (see \cref{sec:discussion} for more details).
\citet{ash21investigating} also provide an example setting of two representations such that the relative order of the NCE loss of these representations changes with $k$, which shows that NCE loss is not consistent about which representation is better when we vary $k$. However, this example does not consider representations that minimizes the NCE loss.
Moreover, the upper bounds in these works hold for {\em all} representations, whereas the representation of interest are only the ones found by minimizing a loss function such as the one defined in Eq. \eqref{eq:contrastive-loss-logistic}, which we refer to as an {\em NCE optimal representation}.

Inspired from the above work, we study the following fundamental question:
\OnlyICML{\vspace{-3mm}}
\begin{center}
	{\em Do more negative samples necessarily hurt the downstream performance of the \textbf{NCE optimal representations}?}
\end{center}

\paragraph{Our Contributions.}
To answer the above question we study the framework of contrastive learning with latent classes as introduced by \citet{saunshi2019theoretical} (also studied by \cite{ash21investigating}). Under this model we show the following results.
\begin{itemize}[leftmargin=*,itemsep=0pt]
	\item In \cref{sec:structural}, we obtain structural results characterizing the NCE optimal representation under certain assumptions on the data distribution. In particular, in \cref{thm:latent-ind-rep}, we show that when the latent classes are non-overlapping, the optimal NCE representation maps all points in the same class to the same vector and points in different classes map to different vectors. Moreover, we give a precise characterization of the NCE optimal representation in a specific setting where the distribution over latent classes is uniform (\cref{thm:simplex-etf}), and show that its performance on the downstream classification task in fact {\em does not degrade with increasing $k$}. We conjecture that this holds even in the case of a non-uniform distribution over latent classes.
	\item In \cref{sec:population-experiments}, we show empirical evidence towards our conjecture using numerical simulations of NCE optimal representations and their corresponding supervised learning loss. In order to do so, we use our structural results to formulate the task of finding the NCE optimal representation as a tractable convex optimization problem.
	\item In \cref{sec:cifar-experiments}, we corroborate our structural characterization results with experiments on the CIFAR10 and the CIFAR 100 datasets \citep{krizhevsky2009learning} which show that to a large extent the structural properties in \cref{thm:simplex-etf} we showed for the minimizer of the population NCE loss hold true on real data.
	\OnlyICML{\vspace{-2mm}}
\end{itemize}
While our assumptions are admittedly restrictive and does not correspond to practical use-cases of contrastive learning, our main goal is to shed light on the ``collision-coverage'' trade-off in this simplified example. Our observations suggest that the ``collision-coverage'' trade-off is not {\em inherent} in contrastive learning and perhaps the phenomena of more negative samples hurting downstream performance has more to do with other aspects of a contrastive learning algorithm, such as the implicit bias of optimizing with gradient based methods, generalizing from finite samples, choice of network architecture, etc.

\section{Contrastive Learning with Latent Classes}\label{sec:prelims}

We consider the following theoretical framework of {\em latent classes}, as introduced by \citet{saunshi2019theoretical} and also studied by \citet{ash21investigating}. Let $\calC$ be a set of latent classes with $|\calC| =: C$. With each latent class $c \in \calC$ we will associate a distribution $\calD_c$ over the {\em input space} $\calX$, which we view as the distribution over data conditioned on belonging to latent class $c$. We will also assume a distribution $\rho$ on $\calC$. We let $\calD$ be the mixture distribution obtained by sampling an input $x \sim \calD_c$ for a class $c \sim \rho$.

We assume access to similar data points in the form of pairs $(x, x^+)$ and $k$ negative data points $x_1^-, \ldots, x_k^-$. To formalize this, an unlabeled sample from $\calD_{\mathrm{NCE}}$ is generated as follows:\OnlyICML{\vspace{-3mm}}
\begin{itemize}[leftmargin=*,itemsep=0pt]
	\item Sample class $c \sim \rho$ and draw i.i.d. samples $x, x^+ \sim \calD_c$.
	\item Draw $x_i^-$ according to $\calD$ for $i \in \set{1, \ldots, k}$.
	\item Return $(x, x^+, x_1^-, \ldots, x_k^-)$.
\end{itemize}

\paragraph{NCE objective.}
The goal of contrastive learning is to learn a good representation using unsupervised data; we consider the set of representations $f : \calX \to \bbS^{d-1}$ that map the input to unit vectors in $d$ dimensions. This is done using the following objective, which intuitively encourages representations of similar inputs to be close to each other, and distinguishes it from representations of random inputs.
\begin{definition}\label{def:nce-loss}
	The NCE loss for a representation $f$ on the distribution $\DNCE$ is defined as\footnote{We use the notation $\tuple{a_i}_{i=1}^k$ to denote the tuple $(a_1, \ldots, a_k)$.}
	\[\LNCEk(f) ~:=~ \Ex_{\DNCE} \insquare{\ell\inparen{\tuple{f(x)^\top (f(x^+) - f(x_i^-))}_{i=1}^k}}\,.\]
	The empirical NCE loss with a finite set $S$ of samples $(x, x^+, x_{1:k}^-)$ drawn from $\DNCE$ is
	\[\textstyle\what{\calL}_{\mathrm{NCE}}^{(k)}(f) ~:=~ \frac{1}{|S|} \sum\limits_{S} \ell\inparen{\tuple{f(x)^\top (f(x^+) - f(x_i^-))}_{i=1}^k} \]
\end{definition}
We restrict $\ell$ to be one of two standard loss functions, hinge loss $\hingeloss^\beta(v) = \max\set{0, \max_i \set{1 -\beta v_i}}$ and logistic loss $\logloss^\beta(v) = \log(1 + \sum_i \exp(-\beta v_i))$, where $\beta$ is a scale (or ``inverse-temperature'') parameter; we often drop the superscript of $\beta$; all our theorems hold for all values of $\beta$. Note that both these losses are convex in $v$ and non-increasing in each $v_i$. Both of these losses have been used in practical NCE implementations \citep{schroff15facenet,chen2020simple}.

Our goal in this paper is to understand the role of negative samples in the NCE loss, in the performance of the representation in downstream supervised learning tasks. While algorithms in practice aim to minimize the empirical NCE loss, we focus on understanding the role of negative samples theoretically at the {\em population} level, namely, we consider both the population NCE loss as well as the population supervised learning loss. This allows us to bypass the issue of generalizing from finite samples and fundamentally understand the role of negative samples.

\paragraph{Downstream supervised learning task.}
We consider the performance of a representation as measured on the downstream supervised learning task of classifying a data point into one of the classes in $\calC$ using a linear predictor over the representation. In particular, let $\Dsup$ be the distribution over $(x,c)$ obtained by sampling $c \sim \rho$ and $x \sim \calD_c$.
\begin{definition}\label{def:sup-loss}
For any representation $f : \calX \to \bbS^{d-1}$, the supervised learning loss is given as\info{PK: Note that the loss $\ell$ in $\LNCEk$ takes in $k$ arguments, whereas here $\ell$ takes $|\calC|-1$ arguments. Perhaps that is why \citet{saunshi2019theoretical} studied the case where we only restrict to downstream tasks with $k+1$ classes.}
\[
\Lsup(f) ~:=~ \inf_{\substack{\set{w_c \mid c \in \calC} \\ \|w_c\|_2 \le 1}} \Lsup(f, \tuple{w_c}_{c \in \calC})
\]
where $\Lsup(f, \tuple{w_c}_{c \in \calC})$ is given as
\[
\Lsup(f, \tuple{w_c}_{c \in \calC}) := \Ex\limits_{(x, c) \sim \Dsup} \ell\inparen{\tuple{f(x)^\top (w_{c} - w_{c'})}_{c' \ne c}}
\]
\end{definition}

This downstream task is exactly the same as the one considered by \citet{ash21investigating}. On the other hand, \citet{saunshi2019theoretical} consider a a slightly different downstream task of classifying into $k+1$ classes (sampled from $\rho$). We go with above formulation as it disentangles the number of negatives in the NCE loss from the downstream task, and moreover allows for the number of negatives to be arbitrarily large (even more than the number of latent classes). However, this means that $\Lsup$ in a sense has a ``different scale'' than $\LNCEk$, and any direct comparison of the two kinds of losses has to adjust for the scale.\unsure{PK: Please see if this last line is okay.} \nishanth{Looks good.}

Note that we consider the restriction of $\|w_c\|_2 \le 1$, to have some sense of scale.\footnote{Without such a restriction, $\Lsup(f)$ can be $0$ for trivial reasons. For example, if the class marginals $\calD_c$ have disjoint supports, then for any $f$ that maps points of different classes to different vectors has $\Lsup(f) = 0$ for both the logistic loss and the hinge loss of $0$ by scaling up $\|w_c\|$ arbitrarily.} The constant $1$ is arbitrary as it is interchangeable with the scale parameter $\beta$ in the loss function.



\subsection{Related Work}
\label{sec:related}
Unsupervised representation learning has a long and rich history including the study of classical methods such as clustering \citep{coates2012learning}, dictionary learning  and non-negative matrix factorization \citep{mairal2009online, pennington2014glove, lee1999learning}, and modern deep learning based techniques such as contrastive learning \citep{chen2020simple} and masked language modeling \citep{devlin2018bert}. Here we discuss the works most relevant to our setting.

While contrastive learning has shown impressive empirical performance in recent years \citep{chen2020simple}, its effectiveness is poorly understood from a theoretical perspective. \citet{wang2020understanding} present a theoretical study of contrastive learning under certain assumptions on the data distribution
showing that asymptotically (as $k \to \infty$) the NCE optimal representation balances a trade-off between being uniformly distributed on the hypersphere and a property called {\em alignment}: the learnt representations of a positive pair $(x, x^+)$ are close to each other.
\citet{zimmermann2021contrastive} show that under a natural data generation model involving latent variables, optimizing the NCE loss corresponds to a form of non-linear independent component analysis (ICA) and the learnt representations can disentangle the latent space.
\citet{von2021self} study the data augmentation process in contrastive learning, i.e., the process of generating positive pairs and negative examples; assuming that the feature space consists of a {\em content} part that is invariant to augmentations and a {\em style} part. They show that optimizing the NCE loss can learn to separate these parts of the feature representations.

The closest to our work are the theoretical results of \citet{saunshi2019theoretical} and the recent work of \citet{ash21investigating}. \citet{saunshi2019theoretical} proposed a natural model for contrastive learning and provided upper bounds on the supervised loss of a representation $f$ in terms of the bound on the NCE loss of $f$. \citet{ash21investigating} further improved this upper bound, with sharper theoretical analysis. Based on the provided upper bounds these results indicate that the performance of a representation learned via contrastive learning can degrade with $k$, the number of negative samples, beyond a certain point. Working in the same model, our results show that the full picture is more subtle and if one could exactly optimize the NCE loss then the degradation with $k$ may not occur at all. In particular, our main result shows that under certain uniformity assumptions, the NCE optimal representation corresponds to the simplex ETF structure (with perfect downstream classification accuracy), that has also been observed in representations learned via standard supervised learning \cite{papyan2020prevalence}. 
After the publication of our work, we were made aware of the works of \citet{bao2021sharp,nozawa2021understanding} which improve the bounds from \cite{ash21investigating} and support our message that increasing the number of negative samples need not hurt downstream classification performance. Our techniques differ from these works though.
In \cref{sec:discussion} we provide a more detailed discussion of our results in the context of the results of \citet{saunshi2019theoretical} and \citet{ash21investigating}. We also discuss the relation of our results with those of \cite{nozawa2021understanding,bao2021sharp} in \cref{sec:discussion}.

\citet{haochen2021provable} relax certain conditional independence assumptions made in prior works \citep{saunshi2019theoretical} and design a new ``spectral contrastive loss'' function that has similarities to the traditional NCE loss, and show that optimizing the new loss using techniques from spectral graph theory can lead to near optimal downstream accuracy under certain assumptions.
\citet{saunshi22understanding} argue that in practically relevant settings, the distribution of the positive example $x^+$ corresponding to different $x$'s have little to no overlap, and explaning the success of the representations learnt on downstream supervised tasks cannot be done without accounting for specific inductive biases in the contrastive learning procedure; note that the setting we consider does not fall in this regime since any two inputs in the same latent class have the same distribution of positive examples.
Finally, there has also been recent work exploring whether contrastive learning can be performed without the use of negative samples while avoiding the phenomenon of feature collapse \citep{tian2021understanding, grill2020bootstrap}.


\section{Structural Results}\label{sec:structural}

We prove structural results about the representation $f : \calX \to \bbS^{d-1}$ that minimizes the (population) $\LNCEk(\cdot)$ loss. We do so under some simplifying assumptions about the set of class distributions $\set{\calD_c}_{c \in \calC}$ and the distribution $\rho$ over classes $\calC$.

\subsection{Non-overlapping Latent Classes}\label{subsec:non-overlap}

Our first structural result considers the case when the distributions $\calD_c$ have mutually disjoint supports.

\begin{assumption}[Non-Overlapping Latent Classes]\label{ass:non-overlap}
	The distributions $\set{\calD_c : c \in \calC}$ have mutually disjoint supports. In this case, we let $c(x)$ denote the unique latent class $c$ such that $x$ lies in the support of $\calD_c$.
\end{assumption}

We show that under this assumption, there exists an optimal representation that maps all points in the support of $\calD_c$ to the same representation vector, whenever the loss satisfies certain simple properties (both logistic and hinge losses satisfy these conditions).

\begin{property}\label{def:sub-additive-loss}
	For a loss function $\ell : \bbR^{t} \to \bbR_{\ge 0}$ it holds for all subsets $S \subseteq \set{1, \ldots, t}$ and $v \in \bbR^t$ that\OnlyICML{\vspace{-2mm}}
	\[
	\textstyle\ell(v) \ge \frac{1}{|S|} \cdot \sum\limits_{j \in S} \ell(v^{S \gets j}) \ \ \text{where, }\ v^{S \gets j}_i := \infork{v_i & \text{if } i \notin S \\ v_j & \text{if } i \in S}.
	\]
\end{property}

In words, a loss $\ell$ satisfies \cref{def:sub-additive-loss} if for all inputs $v$ and all subsets $S$ of the coordinates, substituting all coordinates in $S$ by $v_j$ for some uniformly random $j \in S$ on average does not increase the loss.

\begin{restatable}{observation}{obsSubadditive}\label{obs:logistic-hinge-subadditive}
	$\forall \beta > 0$ : $\logloss^\beta$ and $\hingeloss^\beta$ satisfy \cref{def:sub-additive-loss}.
\end{restatable}

\cref{obs:logistic-hinge-subadditive} is proved in \cref{apx:latent-ind-rep-proofs}.\footnote{Not all convex functions satisfy \cref{def:sub-additive-loss}; e.g. $\ell(v_1 + v_2) = (v_1 + v_2)^2$ is convex, but violates \cref{def:sub-additive-loss} for $S = \set{1, 2}$ and $v = (0, 1)$.} Our structural result is stated using the notion of {\em latent-indistinguishable} representations.
\begin{definition}
	Under \cref{ass:non-overlap}, a representation $f : \calX \to \bbS^{d-1}$ is said to {\em latent-indistinguishable} if $f(x) = f(x')$ for all $x, x'$ satisfying $c(x) = c(x')$. Similarly, $f$ is said to be {\em almost latent-indistinguishable} if $\Pr_{c \sim \rho\,,x, x' \sim \calD_c}[f(x) = f(x')] = 1$.
\end{definition}

Our first structural result shows that under the assumption of non-overlapping classes, there exists a latent-indistinguishable representation that minimizes the population NCE loss.

\begin{restatable}{theorem}{thmLatentIndep}\label{thm:latent-ind-rep}
	Under \Cref{ass:non-overlap}, for any convex, non-increasing loss $\ell$ satisfying \cref{def:sub-additive-loss}, it holds for all representations $f : \calX \to \bbS^{d-1}$, that there exists a latent-indistinguishable representation $\wtilde{f} : \calX \to \bbS^{d-1}$ such that $\LNCEk(\wtilde{f}) \le \LNCEk(f)$. Moreover, if $\ell$ is strictly convex (e.g. logistic loss), then the inequality above is strict, unless $f$ is almost latent-indistinguishable.\info{PK: I couldn't prove tightness of hinge loss, could only do it for strictly convex losses.}
\end{restatable}

{\em Proof Sketch.}
We show the existence of $\wtilde{f}$ in an existential manner. For a fixed latent class $c^* \in \calC$, sample $x^* \sim \calD_{c^*}$ and define a representation $f_{x^*} : \calX\to \bbS^{d-1}$ that maps all inputs in the same class as $x^*$ to $f(x^*)$ leaving all other representations intact, namely,
\begin{align*}
	f_{x^*}(x) &~:=~ \begin{cases}
		f(x^*) & \text{if } c(x) = c^*\\
		f(x) & \text{if } c(x) \ne c^*
	\end{cases}
\end{align*}
We show that
\begin{align}
	\Ex_{x^* \sim \calD_c} \insquare{\LNCEk(f_{x^*})} ~\le~ \LNCEk(f)\,.\label{eq:latent-ind-rep-1}
\end{align}
This implies the existence of an $x^*$ in support of $\calD_{c^*}$ such that $\LNCEk(f_{x^*}) \le \LNCEk(f)$. Iteratively repeating this argument for each latent class $c^* \in \calC$, shows the existence of a latent-indistinguishable $\wtilde{f}$ with $\LNCEk(\wtilde{f}) \le \LNCEk(f)$.

In order to show \eqref{eq:latent-ind-rep-1}, we consider two cases depending on the latent class of the sampled positive pair $(x, x^+)$: (i) $c(x) = c(x^+) = c^*$ and (ii) $c(x) = c(x^+) \ne c^*$. In case (i), the argument holds for any convex and non-increasing loss. In case (ii), the argument holds for any loss satisfying \cref{def:sub-additive-loss}. We defer the proof details to \cref{apx:latent-ind-rep-proofs}.

\subsection{Uniform distribution over latent classes}\label{subsec:uniform-class-dist}

Our second structural result considers the case when in addition to non-overlapping latent classes, the distribution $\rho$ over the latent classes is uniform.
\begin{assumption}[Uniform Latent Classes]\label{ass:uniform}
	$\rho$ is uniform over the latent classes $\calC$.
\end{assumption}
Here, we show that the NCE optimal representations are precisely characterized by Simplex ETFs \cite{van1966equilateral,papyan2020prevalence}.
\begin{definition}[Simplex ETF Representation]\label{def:simplex-etf}
	Under \Cref{ass:non-overlap}, $f : \calX \to \bbS^{d-1}$ is a {\em Simplex Equiangular Tight Frame} (Simplex ETF) representation for a distribution $\calD$, if the following conditions hold:\OnlyICML{\vspace{-3mm}}
	\begin{itemize}[leftmargin=*,itemsep=0pt]
		\item $f$ is {\em latent-indistinguishable}, and
		\item $f(x)^\top f(x') = -1/(C-1)$ for all $x, x'$ s.t. $c(x) \ne c(x')$.\OnlyICML{\vspace{-3mm}}
	\end{itemize}
	$f$ is an {\em almost Simplex ETF} representation if there exists a Simplex ETF representation $f'$ such that $\Pr_{x \sim \calD}[f(x) = f(x')] = 1$.
\end{definition}

A latent-indistinguishable representation is said to be {\em equiangular} if $f(x)^\top f(x') = \alpha$ for all $x, x'$ such that $c(x) \ne c(x')$, for some value of $\alpha$. A Simplex ETF representation achieves the smallest value of $\alpha$, among all equiangular representations. Our second structural result shows that under the assumption of non-overlapping and uniform latent classes, Simplex ETF representations are NCE optimal.
\begin{theorem}\label{thm:simplex-etf}
	Under \cref{ass:non-overlap,ass:uniform}, any Simplex ETF representation $f$ minimizes $\LNCEk(f)$ for any convex and non-increasing loss $\ell$ satifying \cref{def:sub-additive-loss}.
	Moreover, if $\ell$ is also strictly convex (e.g. logistic loss), then (almost) Simplex ETF representations are the only minimizers of $\LNCEk(f)$.
\end{theorem}

\noindent \Cref{thm:simplex-etf} follows immediately from combining \cref{thm:latent-ind-rep} with the following claim.

\begin{restatable}{claim}{claimEquiangular}\label{claim:equiangular}
	Under \cref{ass:non-overlap,ass:uniform}, for any convex, non-increasing loss $\ell$, it holds for all (almost) latent-indistinguishable representations $f : \calX \to \bbS^{d-1}$, that $\LNCEk(\wtilde{f}) \le \LNCEk(f)$ for any Simplex ETF $\wtilde{f} : \calX \to \bbS^{d-1}$.
	Moreover if $\ell$ is strictly convex (e.g. logistic loss), then equality holds only if $f$ is an (almost) Simplex ETF representation. \info{PK: I couldn't prove tightness of hinge loss, could only do it for strictly convex losses.}
\end{restatable}

{\em Proof Sketch.}
Let $u_c$ denote the (common) representation for all $x$ in latent class $c$, namely $f(x) = u_{c(x)}$. Observe that $\|\sum_{c \in \calC} u_c\|_2^2 = C + \sum_{c \ne c'} u_c^\top u_{c'} \ge 0$ and hence $\Ex_{c, c' \sim \rho} [u_c^\top u_{c'} \mid c \ne c'] \ge -1/(C-1)$ (under \Cref{ass:uniform} that $\rho$ is uniform over $\calC$). Let $\wtilde{f}$ be an {\em equiangular} representation given as $\wtilde{f}(x) = \wtilde{u}_{c(x)}$ satisfying
\[
\wtilde{u}_c^\top \wtilde{u}_{c'} = \begin{cases} 1 & \text{if } c = c'\\ \Ex_{c, c' \sim \rho}[u_c^\top u_{c'} \mid c \ne c'] & \text{if } c \ne c'\end{cases},
\]
We show that $\LNCEk(\wtilde{f}) \le \LNCEk(f)$ for any convex loss $\ell$ (via Jensen's inequality). Finally, for any non-increasing loss $\ell$, any Simplex ETF minimizes $\LNCEk(\cdot)$ among all equiangular representations, as it achieves the smallest value of $u_c^\top u_{c'}$. We defer the proof details to \cref{apx:proof-equiangular-proof}.

\subsection{Downstream performance of NCE Optimal Representations}\label{subsec:downstream-conjecture}

From \cref{thm:simplex-etf}, we have that the NCE optimal representation in the case of non-overlapping latent classes with uniform distribution over them, in fact {\em does not depend} on $k$, the number of negatives in $\LNCEk(\cdot)$. And hence, for $f_k := \argmin_{f} \LNCEk(f)$, it holds that $\Lsup(f_k)$ is independent of $k$.
\Experimental{ 
More strongly, the NCE optimal representation turns out to also be the optimal representation minimizing $\Lsup$ (proof in \cref{apx:supervised-optimal-proof}).
\begin{restatable}{theorem}{thmSupervisedOptimal}\label{thm:supervised-optimal}
	Under \cref{ass:non-overlap,ass:uniform}, any Simplex ETF representation $f$ minimizes $\Lsup(f)$ for any convex, non-increasing loss $\ell$ satisfying \cref{def:sub-additive-loss}.
\end{restatable}
}{\info{It would be nice to show that in fact Simplex ETF is a minimizer of $\Lsup$ in this case.}}%
But what about the case where the marginal over the latent classes is not uniform? We conjecture that just under \cref{ass:non-overlap}, the supervised learning loss of the NCE optimal representation is non-increasing in $k$.
\begin{conjecture}\label{conj:sup-loss-non-increasing}
	For all $C \ge 3$, under \cref{ass:non-overlap}, for all distributions $\rho$ over classes $\calC$: $\Lsup(f_k)$ is non-increasing in $k$, where $f_k := \argmin_f \LNCEk(f)$.
\end{conjecture}

Note that the statement of the conjecture holds trivially for $C=2$, even under non-uniform distribution $\rho$, since in this case, $f_k$ is a latent-indistinguishable representation that maps points of the two classes to anti-podal points. Thus, $f_k$ is the same for all $k$ and hence $\Lsup(f_k)$ is non-increasing (in fact constant) in $k$.

While we are unable to prove this conjecture formally, we provide empirical evidence towards this conjecture in \cref{sec:population-experiments} by numerically computing NCE optimal representations.

\section{Empirical Evaluation of the Supervised Loss of NCE Optimal Representations}\label{sec:population-experiments}

The high dimensional space of representations poses a challenge for numerically computing the NCE optimal representation $\LNCEk(\cdot)$ for any $k$. However, under \cref{ass:non-overlap}, using our structural result (\cref{thm:latent-ind-rep}), we can reformulate the goal as a convex optimization problem over a small number of variables.

For any latent-indistinguishable representation, let $u_c$ denote the (common) representation for all $x$ in latent class $c$, namely $f(x) = u_{c(x)}$. Let $Z \in \bbR^{C \times C}$ be a matrix given by $Z_{c,c'} := u_{c}^\top u_{c'}$ for $c, c' \in \calC$. Note that $Z$ encodes a latent-indistinguishable representation if and only if $Z$ is a {\em correlation matrix}, that is, Z is positive semi-definite with $Z_{c,c} = 1$ for all $c \in \calC$.

The key observation is that $\LNCEk(f)$ is convex when formulated in terms of $Z$:
\begin{align}
	\LNCEk(f) ~=~ \Ex_{c,c_{1:k} \sim \rho} \insquare{\ell(\tuple{1 - Z_{c,c_i}}_{i=1}^k)}\label{eq:NCE-convex-formulation}
\end{align}

\paragraph{Details of the Numerical Simulation.}
In order to empirically validate our proposed \cref{conj:sup-loss-non-increasing} we vary the number of classes $C \in \{3, 9\}$. For each $c \in \calC$, we generate a class distribution drawn from a Dirichlet prior with a uniform parameter $\alpha$. We vary alpha from $1$ to $4$ to get four class distributions for each value of $C$.

Next for each $C$ and class distribution 
we numerically estimate the minimizer $Z$ of \eqref{eq:NCE-convex-formulation} by performing projected stochastic gradient descent over the space of $C \times C$ correlation matrices, with a decaying step size (as suggested by \citet{lacostejulien12simpler}). We sample stochastic gradients by averaging the gradients $\nabla_Z\, \ell(\tuple{1 - Z_{c,c_i}}_{i=1}^k)$ over independently sampled mini-batches of $(c, c_{1:k}) \sim \rho^{k+1}$ at each step. The projection to the space of correlation matrices is done via the algorithm of \citet{higham02computing}. In our experiments, we fix the mini batch size to be $10000$, and perform $1000$ steps of projected gradient descent with an initial step size of $50$. We compute the mean of the $Z$ matrices computed across $400$ independent runs of projected gradient descent and then extract the per class embeddings $\tuple{u_c}_{c \in \calC}$ via a Cholesky decomposition of the mean matrix. Finally, we optimize the class weighted logistic loss over $\tuple{w_c}_{c \in \calC}$ to compute the value of the supervised loss. 

In \cref{fig:simulated} we plot for each $c \in C$, the supervised loss as a function of $k$ for across values of $\alpha$. We find that the downstream supervised loss obtained via the NCE optimal representation is essentially non-increasing in $k$. In particular, 22 out of the 28 curves in \cref{fig:simulated} are strictly decreasing in $k$. However, the remaining 6 curves are non-monotonic at some values of $k$, which seems to contradict our conjecture. But we suspect that this is likely because of imprecision in our numerical simulation procedure.\unsure{PK: Double check this paragraph.}


\tikzset{
	curvethick/.style = {line width=0.5pt},
	colorThree/.style = {curvethick,Red,mark=star},
	colorFour/.style = {curvethick,Cerulean,mark=o},
	colorFive/.style = {curvethick,Green,mark=*},
	colorSix/.style = {curvethick,Orange,mark=diamond*},
	colorSeven/.style = {curvethick,Purple,mark=triangle*},
	colorEight/.style = {curvethick,Gray,mark=square*},
	colorNine/.style = {curvethick,Dandelion,mark=pentagon*}
}
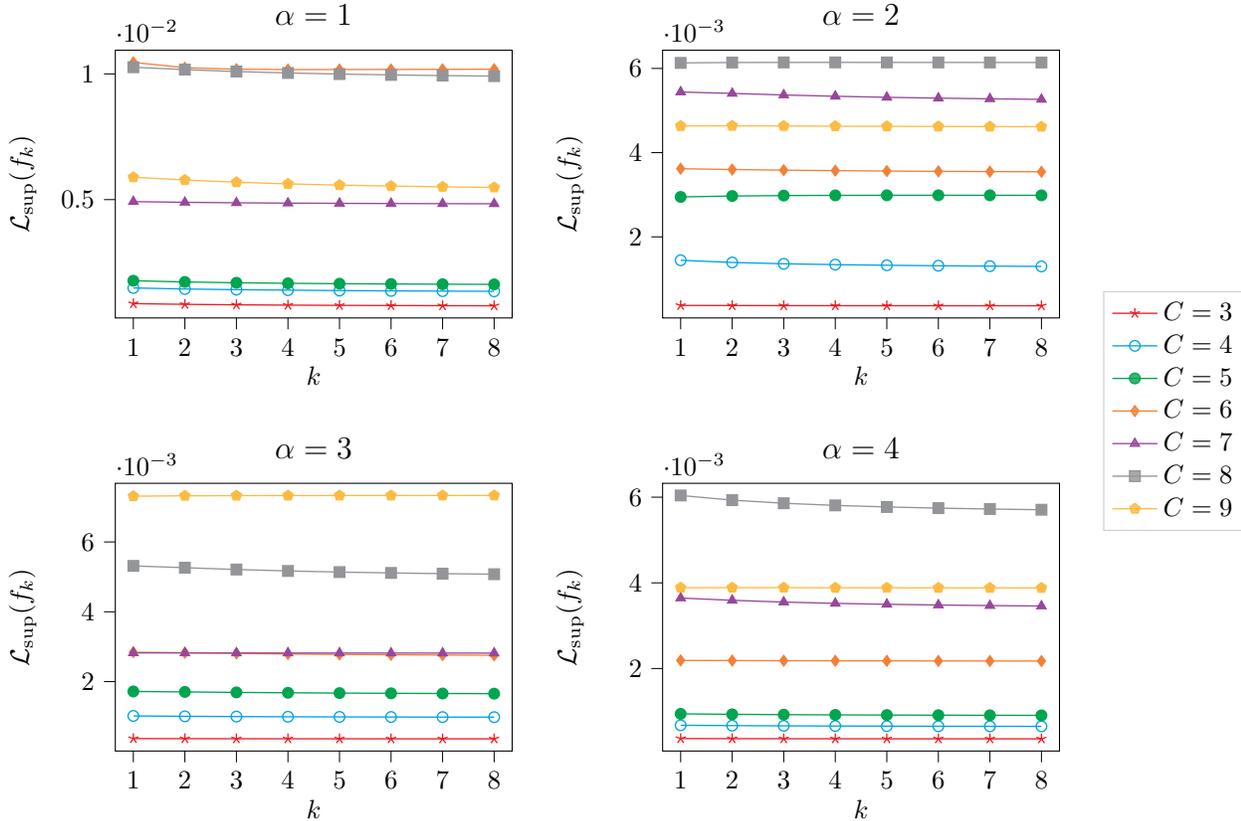
\begin{figure*}[htb]   
	\centering
\begin{tikzpicture}
	
	\definecolor{darkgray176}{RGB}{176,176,176}
	\definecolor{green}{RGB}{0,128,0}
	\definecolor{lightgray204}{RGB}{204,204,204}
	\definecolor{orange}{RGB}{255,165,0}
	\definecolor{purple}{RGB}{128,0,128}
	\definecolor{violet}{RGB}{238,130,238}
	
	\begin{groupplot}[group style={group size=2 by 2}]
		\nextgroupplot[
		legend cell align={left},
		legend style={fill opacity=0.8, draw opacity=1, text opacity=1, at={(2.85,0.1)}, draw=lightgray204},
		width=0.4\textwidth,
		height=0.3\textwidth,
		tick align=outside,
		tick pos=left,
		title={\large \(\displaystyle \alpha=1\)},
		x grid style={darkgray176},
		xlabel={\(\displaystyle k\)},
		xmin=0.65, xmax=8.35,
		xtick={1,2,3,4,5,6,7,8},
		xtick style={color=black},
		y grid style={darkgray176},
		ylabel={\(\displaystyle \mathcal{L}_{\mathrm{sup}}(f_k)\)},
		ymin=0.000293929736597684, ymax=0.0109415870933274,
		ytick style={color=black}
		]
		\addplot [colorThree]
		table {%
			1 0.000864264386413986
			2 0.000835022195896602
			3 0.000816327002047318
			4 0.000803625188624551
			5 0.000794474430284652
			6 0.000787567850171522
			7 0.000782206654454238
			8 0.000777914161903579
		};
		\addlegendentry{$C=3$}
		\addplot [colorFour]
		table {%
			1 0.00148657842381933
			2 0.00144471809628792
			3 0.00141876080527398
			4 0.00139859625520036
			5 0.0013824416081181
			6 0.00136987269383008
			7 0.00136005210736454
			8 0.00135229883567101
		};
		\addlegendentry{$C=4$}
		\addplot [colorFive]
		table {%
			1 0.00177407384646262
			2 0.00172515182938503
			3 0.00169315273877784
			4 0.00167133116885483
			5 0.00165578761707156
			6 0.00164427546725989
			7 0.00163545830318834
			8 0.00162852029569043
		};
		\addlegendentry{$C=5$}
		\addplot [colorSix]
		table {%
			1 0.0104576026680215
			2 0.0102463942678871
			3 0.0101871948530334
			4 0.0101720158410211
			5 0.0101737127358902
			6 0.0101762914621283
			7 0.0101806790750775
			8 0.0101863648638138
		};
		\addlegendentry{$C=6$}
		\addplot [colorSeven]
		table {%
			1 0.00492069849059385
			2 0.00489280728451543
			3 0.00487359206588965
			4 0.00486026529656064
			5 0.00485126174458461
			6 0.00484481162660876
			7 0.00483925833848664
			8 0.00483486447568337
		};
		\addlegendentry{$C=7$}
		\addplot [colorEight]
		table {%
			1 0.010264725708124
			2 0.0101709132899266
			3 0.0100942661802395
			4 0.0100371186963399
			5 0.00999327511798791
			6 0.009960057880492
			7 0.00993272866642074
			8 0.00991148220364673
		};
		\addlegendentry{$C=8$}
		\addplot [colorNine]
		table {%
			1 0.00589064543184291
			2 0.00578053223403851
			3 0.00569194166196142
			4 0.00562604260381486
			5 0.00557677472433324
			6 0.00553881328416977
			7 0.00550886929286312
			8 0.0054848472236539
		};
		\addlegendentry{$C=9$}
		
		\nextgroupplot[
		width=0.4\textwidth,
		height=0.3\textwidth,
		tick align=outside,
		tick pos=left,
		title={\large \(\displaystyle \alpha=2\)},
		x grid style={darkgray176},
		xlabel={\(\displaystyle k\)},
		xmin=0.65, xmax=8.35,
		xtick={1,2,3,4,5,6,7,8},
		xtick style={color=black},
		xshift=1cm,
		y grid style={darkgray176},
		ylabel={\(\displaystyle \mathcal{L}_{\mathrm{sup}}(f_k)\)},
		ymin=7.37246931748032e-05, ymax=0.00643267966582801,
		ytick style={color=black}
		]
		\addplot [colorThree]
		table {%
			1 0.000371217996037455
			2 0.000368327973631298
			3 0.000366444165277666
			4 0.000365170752418666
			5 0.000364278506540397
			6 0.000363630643956372
			7 0.000363145338113208
			8 0.000362768101022676
		};
		\addplot [colorFour]
		table {%
			1 0.00144394154211838
			2 0.00139177216926708
			3 0.00136032693654402
			4 0.00133990698492081
			5 0.00132485936376717
			6 0.00131342880799431
			7 0.00130454303292699
			8 0.00129750337220637
		};
		\addplot [colorFive]
		table {%
			1 0.00294819895739079
			2 0.00296997423038911
			3 0.00298053638153192
			4 0.00298506318053457
			5 0.00298720514233663
			6 0.00298729842703054
			7 0.00298721062348384
			8 0.00298629872612769
		};
		\addplot [colorSix]
		table {%
			1 0.00361813444172853
			2 0.00359942979771961
			3 0.00358413910168721
			4 0.00357236554532294
			5 0.00356325677324955
			6 0.0035560610365697
			7 0.0035501950855864
			8 0.00354542915575832
		};
		\addplot [colorSeven]
		table {%
			1 0.00544052328353195
			2 0.00540664672976453
			3 0.00536962836674585
			4 0.00533919581463997
			5 0.00531398842446081
			6 0.00529432836941566
			7 0.0052783849652969
			8 0.00526488835900236
		};
		\addplot [colorEight]
		table {%
			1 0.0061317337566912
			2 0.00614086427977613
			3 0.00614335865327292
			4 0.00614363625798013
			5 0.00614313674779683
			6 0.0061425060085096
			7 0.00614193084927158
			8 0.00614120404686543
		};
		\addplot [colorNine]
		table {%
			1 0.00463208663070788
			2 0.004636599467564
			3 0.00463278799436577
			4 0.00462832039865244
			5 0.00462445360919347
			6 0.00462100160433923
			7 0.0046180678076115
			8 0.0046156764363598
		};
		
		\nextgroupplot[
		width=0.4\textwidth,
		height=0.3\textwidth,
		tick align=outside,
		tick pos=left,
		title={\large \(\displaystyle \alpha=3\)},
		x grid style={darkgray176},
		xlabel={\(\displaystyle k\)},
		xmin=0.65, xmax=8.35,
		xtick={1,2,3,4,5,6,7,8},
		xtick style={color=black},
		y grid style={darkgray176},
		yshift=-1.2cm,
		ylabel={\(\displaystyle \mathcal{L}_{\mathrm{sup}}(f_k)\)},
		ymin=1.18727123259211e-05, ymax=0.0076805268205903,
		ytick style={color=black}
		]
		\addplot [colorThree]
		table {%
			1 0.000368115236614702
			2 0.000365281007131766
			3 0.000363549665965959
			4 0.000362434250061925
			5 0.000361679110284518
			6 0.000361142007244325
			7 0.000360748514350362
			8 0.000360447899065211
		};
		\addplot [colorFour]
		table {%
			1 0.00101675441366445
			2 0.00100612889546227
			3 0.00099862478579906
			4 0.000993174134754626
			5 0.000989124193781512
			6 0.00098601585571626
			7 0.000983553427244453
			8 0.000981609656930568
		};
		\addplot [colorFive]
		table {%
			1 0.00171821686296893
			2 0.00170502242215109
			3 0.00169106485956835
			4 0.00167990570083742
			5 0.00167128446607682
			6 0.00166448881026854
			7 0.00165904718919217
			8 0.00165461759569827
		};
		\addplot [colorSix]
		table {%
			1 0.00284193236231218
			2 0.00282673950085789
			3 0.00280707250624367
			4 0.0027909908318385
			5 0.00277823371075614
			6 0.00276812844885758
			7 0.00275990484031884
			8 0.00275317580818802
		};
		\addplot [colorSeven]
		table {%
			1 0.00282298595158507
			2 0.00282150096334136
			3 0.00282005935283065
			4 0.00281892361635751
			5 0.00281802260861539
			6 0.00281731073816288
			7 0.0028167211605925
			8 0.00281623660588264
		};
		\addplot [colorEight]
		table {%
			1 0.00531397557672074
			2 0.00526290922573675
			3 0.00521088916731194
			4 0.00516984381723941
			5 0.00513781111415951
			6 0.00511277519621419
			7 0.00509271804542137
			8 0.00507643745873421
		};
		\addplot [colorNine]
		table {%
			1 0.00731096793231807
			2 0.00732212191513116
			3 0.00732675030670518
			4 0.00732916324559759
			5 0.00733059393496949
			6 0.00733127085126441
			7 0.00733165982155393
			8 0.00733195163385101
		};
		
		\nextgroupplot[
		width=0.4\textwidth,
		height=0.3\textwidth,
		tick align=outside,
		tick pos=left,
		title={\large \(\displaystyle \alpha=4\)},
		x grid style={darkgray176},
		xlabel={\(\displaystyle k\)},
		xmin=0.65, xmax=8.35,
		xtick={1,2,3,4,5,6,7,8},
		xtick style={color=black},
		y grid style={darkgray176},
		yshift=-1.2cm,
		ylabel={\(\displaystyle \mathcal{L}_{\mathrm{sup}}(f_k)\)},
		ymin=7.39616859733431e-05, ymax=0.00632399561639261,
		ytick style={color=black}
		]
		\addplot [colorThree]
		table {%
			1 0.000364700774666562
			2 0.000362170967386161
			3 0.000360666848197987
			4 0.000359711906757896
			5 0.000359075489958903
			6 0.000358630136316954
			7 0.000358302243236448
			8 0.000358054137356037
		};
		\addplot [colorFour]
		table {%
			1 0.000674582584099435
			2 0.000667144465125756
			3 0.000661719521167213
			4 0.000657712573031242
			5 0.000654665703293399
			6 0.000652290556523762
			7 0.000650393404609916
			8 0.000648874340237777
		};
		\addplot [colorFive]
		table {%
			1 0.000941584643200546
			2 0.000930778040961511
			3 0.000923210824054931
			4 0.000917787817333813
			5 0.000913760657119088
			6 0.000910707092267839
			7 0.000908289568202507
			8 0.000906362407995306
		};
		\addplot [colorSix]
		table {%
			1 0.00218935641617859
			2 0.00218539101761342
			3 0.00218230782885251
			4 0.00217995891286782
			5 0.00217809389592071
			6 0.00217660520032058
			7 0.00217538617028582
			8 0.00217438410563819
		};
		\addplot [colorSeven]
		table {%
			1 0.00364530681945986
			2 0.00359256517777152
			3 0.00355121304292909
			4 0.00352123914144473
			5 0.00349915678841919
			6 0.00348246584566021
			7 0.00346931010605362
			8 0.00345882302338303
		};
		\addplot [colorEight]
		table {%
			1 0.00603990316500992
			2 0.0059304556166394
			3 0.00585824756140518
			4 0.00580804695261791
			5 0.00577170242107844
			6 0.00574474357488625
			7 0.00572346304032526
			8 0.0057067031326289
		};
		\addplot [colorNine]
		table {%
			1 0.00388654006192454
			2 0.00388817631930749
			3 0.0038870777756214
			4 0.00388599448806177
			5 0.00388473594897876
			6 0.00388355323203811
			7 0.00388259655529097
			8 0.00388179653801248
		};
	\end{groupplot}
	
\end{tikzpicture}
	\caption{The figure shows how the supervised loss of the NCE optimal representation varies with $k$ for different multiclass settings corresponding to $C \in \{3,4,\ldots, 9\}$. The four plots correspond to four different parameter settings of the class weights obtained via sampling from a Dirichlet prior with parameter $\alpha$.}
	\label{fig:simulated}
\end{figure*}


\section{Experiments with CIFAR datasets}\label{sec:cifar-experiments}
CIFAR-10 and CIFAR-100 are two well-known image classification benchmark datasets. They are both balanced and contain examples from 10 (100) classes, provide 5000 (500) train examples per class and 1000 (100) test examples per class respectively \cite{krizhevsky2009learning}. We perform experiments of contrastive loss minimization on CIFAR-10 and CIFAR-100 datasets to test to what extent the simplex ETF structure manifests at the end of training. We also measure the downstream classification performance of the learnt classifier on the respective test sets. 

\textbf{Experimental Setting.} We train a ResNet-18/50 architecture with a projection head as our encoder (similar to \citep{chen2020simple}). We use the logistic loss for training and we train for 400 epochs.
We modify the positive pair generation process to match our theoretical setting. Two randomly sampled images with the same label now form a positive pair. We do not apply any other perturbations such as cropping, blurring etc. This allows us to study to what extent our theoretical predictions manifest on complex real data.
Once the encoder is trained, we then train a linear layer for standard downstream classification task of respective dataset.
To measure proximity to the Simplex ETF structure, we record two metrics. (i) First, for each class $c$, we record the mean intraclass variance among representations belonging to $c$ (henceforth referred to as $\IntraVar_c$) and (ii) second, we record the cosine similarities between the mean representation vectors of different classes $c_1, c_2$ (referred to as $\CS(c_1, c_2)$) for all $c_1 \ne c_2$. Formally $\IntraVar_c$ is defined as
\begin{align}
	\IntraVar_c = \sum_{i=1}^{n_c} \frac{1}{n_c}\|r_i - \bar{r} \|_2^2
\end{align}
where $r_i$ is the representation corresponding to the $i^{\text{th}}$ example from class $c$ and $\bar{r} = \sum_{i=1}^{n_c} r_i/n_c$. To visually interpret a particular value of $\IntraVar_c = \alpha$, we use the rough approximation that the angle made by a random representation vector of class $c$ with its class mean is $\arccos((2-\alpha)/2)$. More details on the setting are presented in \cref{apx:cifar-experiments}.

\paragraph{Results on CIFAR-100.}
On CIFAR-100 our model reaches a downstream classification accuracy of $64.76\%$. The expected value of $\CS(c_1, c_2)$ for any $c_1 \ne c_2$ in this setting will be $-1/(C-1) = -1/99 \approxeq 0.01$. We observe a mean $\CS$ value of $0.003 \pm 0.04$ (note that $\CS$ can vary from $[-1, 1]$). The average $\IntraVar_c$ across all 100 classes is $0.52$ which implies that a random representation from class $c$ makes an angle $\approx 42^\circ$. Although not a perfect simplex ETF structure due to the relatively high intra-class variance, this still shows that the inter-class cosine similarities are remarkably close to what is to be expected.

\paragraph{Results on CIFAR-10.} On CIFAR-10 we perform a more extensive set of experiments scaling the values of the number of negative samples $k$. We also sub-sample CIFAR-10 to contain fewer than 10 classes and investigate how the structure of the resulting representations change.
In our setup, we observe a steady increase in the downstream performance all the way from $k=1$ up to $k=512$ (\cref{fig:eval_acc}).
\cref{fig:scaling_c} shows how the average $\CS$ and Intra-Variance$_c$ values change with the total number of classes $C$. The average $\CS$ value in particular is strongly in line with what is predicted by our theory.

\begin{figure}[!htb]   
	\centering
	\OnlyICML{\includegraphics[width=.8\linewidth]{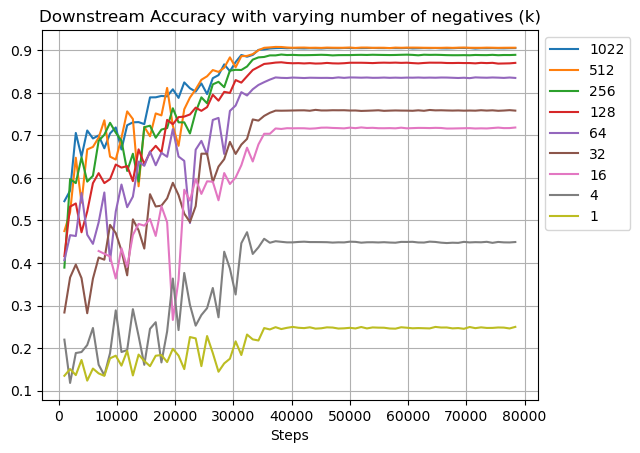}}
	\OnlyArXiv{\includegraphics[width=.45\linewidth]{\imgPath/eval_acc.png}}
	
	\caption{Downstream classification accuracy of contrastively learnt representations on CIFAR-10 improves with increasing the number of negative examples $k$.}
	\label{fig:eval_acc}
\end{figure}

\begin{figure*}[htb]   
	\centering
	\minipage{0.4\textwidth}
	\includegraphics[width=\linewidth]{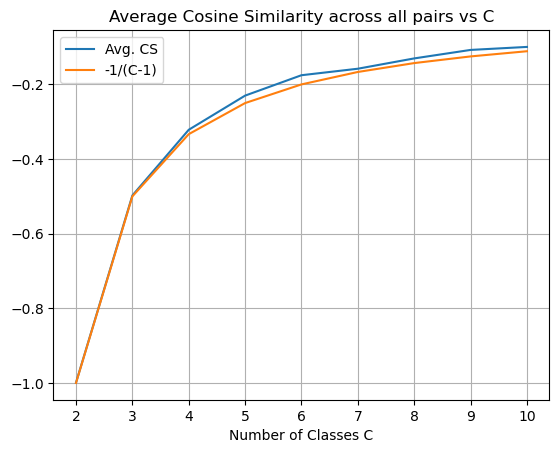}
	\endminipage 
	\minipage{0.4\textwidth}
	\includegraphics[width=\linewidth]{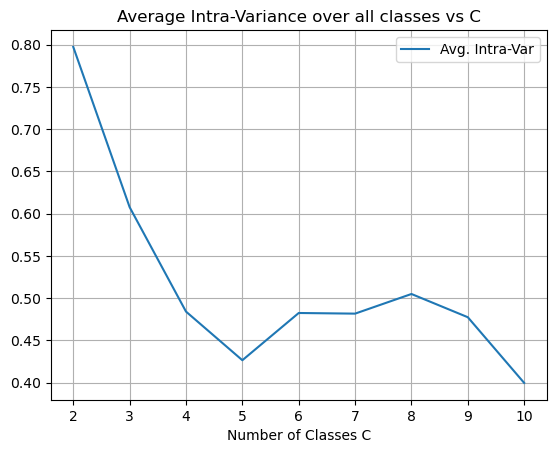}
	\endminipage \hspace{5mm}
	\caption{The change in the average inter-classes cosine similarities and average intraclass variances of the representations is plotted for different values of the number of classes $C$ while the number of negative samples $k$ is fixed at 1022. Experiments were done by sub-sampling CIFAR-10 to have fewer classes. The orange line in the graph on the left plots what our theory predicts.}
	\label{fig:scaling_c}
\end{figure*}

\section{Comparing our Results with Prior Work}\label{sec:discussion}
Here we discuss our results and their implications in context of prior work studying the impact of negative samples in contrastive learning.

\paragraph{Comparing with prior empirical observations.} Some empirical phenomena suggest that while more negative samples help initially, when increased beyond a certain point, they can start to hurt \citep{ash21investigating,mitrovic2020less}. On the other hand, as we observe in our paper, if we analyze the representation minimizing the population NCE loss, larger number of negative samples continues to help improve the downstream supervised learning task, or at least does not hurt it. How do we then reconcile these two observations? 

There are three important ways in which our analysis simplifies what happens in the real world. The first is by assuming that our optimizer (typically a variant of stochastic gradient descent) has reached a global minimum. This might not always hold in practice.
In particular, the optimizer might find it harder to reach a global minima with increasing values of $k$.
The second aspect is that perhaps the class of deep neural nets optimized by SGD is not expressive enough to be able to perfectly satisfy our structural characterization of the minimizer. This corresponds to the setting when there is an approximation error as we minimize among a restricted class of functions. On the other hand, if the class is highly expressive, then there remains the question of finite sample generalization of the empirical NCE optimal representation.
And lastly, the positive pair generation process may be quite different to the one we considered which could lead to quite different properties of the minimizer.

Apart from these three reasons, another factor to consider is our assumption of non-overlapping latent classes which might not hold perfectly in practice leading to small inconsistencies in observed performance.
Moreover, the story is not very clear on the empirical side as well. In works which report experiments showing that a large number of negative samples hurt performance beyond a point \citep{ash21investigating,saunshi2019theoretical}, the degradation in performance is quite small ($1-2\%$ drop in accuracy) and it is unclear if it cannot be attributed to noise introduced during training. \citet{mitrovic2020less} also report a small amount of degradation at higher $k$ but their positive pair generation process is different from what we consider in our setting. In addition, experiments performed in the works of \citet{nozawa2021understanding,bao2021sharp} seem to indicate that increasing $k$ does not hurt downstream classification performance.

\paragraph{Comparing with prior theoretical observations.} 
The theoretical work closest to ours is that of \cite{saunshi2019theoretical,ash21investigating} who work with a similar framework as ours. The main results in these works are upper bounds on the supervised loss of any representation in terms of its NCE loss. That is, they show that, for any representation function $f$,
\begin{align}
    \Lsup(f) \le \alpha(k,\rho) \cdot \inparen{\LNCEk(f) - \tau(k)}\,.
\end{align}
where in the setting of uniform distribution $\rho$ over latent classes, $\alpha(k, \rho) \approxeq \frac{4\log C (C-1) \eta^k}{k}$ for $\eta = 1 + 1/(C-1) > 1$ and hence for large $k$, $\alpha(k, \rho)$ grows exponentially with $k$, and $\tau(k) := 1 - (1 - \frac{1}{C})^k < 1$. However, this is just an upper bound on the downstream performance. 
In contrast, our result shows that for the minimizer $f^*$ of the population NCE loss, $\Lsup(f^*)$ does not increase with $k$ under \cref{ass:uniform} and we give supporting evidence through simulations that this monotonic behavior persists even without the uniformity assumption. Moreover in \Cref{thm:improved-ash} (in \Cref{apx:ash21improvement}), we improve the above result in the case of logistic loss by showing that for \emph{any} $f$, 
$$\Lsup(f) \le \beta(k,\rho) \cdot \LNCEk(f)$$
where $\beta(k,\rho)$ is in fact non-increasing in $k$ and becomes a constant for large enough $k$. We emphasize that while our result does not have the $\tau(k)$ term, it doesn't affect the asymptotic behavior with respect to $k$, since $\tau(k) \le 1$ whereas $\LNCEk(f)$ increases with $k$. However, we note again that $\Lsup$ (which is a classification objective over $C$ classes) has a different scale than $\LNCEk$ (which is a classification objective over $k$ classes). In particular, $\Lsup(f)$ is at most $\log C$ for a (trivial) constant representation, whereas $\LNCEk(f)$ grows as $\sim \log k$ even for the NCE optimal representation. Hence it may not be meaningful to prove an upper bound on the downstream supervised loss of a representation directly in terms of the NCE loss, and instead study the representations that are NCE optimal (or close to one) directly as we do in this paper.

To support their proposition that a worsening of performance with increased $k$ is unavoidable, both \citet{saunshi2019theoretical, ash21investigating} present example representation functions and distributions where the representation giving better NCE loss yields a worse supervised loss. However these are not formal lower bounds on the performance of the \emph{optimal} representation that minimizes the NCE loss. Indeed, our results suggest that such a formal lower bound might not exist as the performance does not degrade with $k$ (\cref{conj:sup-loss-non-increasing}).
In addition, the examples provided in these works sometimes rely on unnormalized representations which is not what works well in practice.

After the publication of our work, we were informed about the works of \citet{nozawa2021understanding,bao2021sharp} which provide sharper theoretical bounds for the supervised loss in terms of the contrastive loss. \citet{bao2021sharp} provide upper and lower bounds on the supervised loss of any representation in terms of the contrastive loss obtained on that representation. In the regime of large $k$, these upper and lower bounds differ by a constant which is independent of $C,k$. We compare and contrast these results with ours across two settings: (i) uniform class distribution: in this setting, our structural characterization in \Cref{thm:simplex-etf} provides an exact understanding of the downstream loss of the NCE optimal representation. On the other hand, the bounds provided in \citet{nozawa2021understanding,bao2021sharp} (which applies to any representation) do not in particular imply that the performance of the NCE optimal representation on the downstream task does not degrade at all with increasing $k$; (ii) non-uniform class distributions: in this setting, we are not able to show that the downstream performance does not degrade with increasing $k$. The bounds in \citet{bao2021sharp} continues to provide a sharp characterization of the supervised loss of any representation, but as far as we know, \Cref{conj:sup-loss-non-increasing} remains open. At a broader level, the message in the works of \cite{nozawa2021understanding,bao2021sharp} is aligned with the message in our work that more negative samples in contrastive learning do not necessarily hurt downstream performance.

To conclude, we analyzed normalized representations and see a much nicer structure emerge in the population NCE optimal representation. Our investigation suggests that the ``collision-coverage'' tradeoff is not sufficient on its own for explaining the non-monotonic behavior in the downstream performance as a function of the number of negative samples that is observed in practice.

\section*{Acknowledgements}
We would like to thank Simon Kornblith for help with running the experiments and related discussions, Sankeerth Rao Karingula for discussions about optimizing over correlation matrices and Kento Nozawa for bringing to our attention the papers of \citet{nozawa2021understanding,bao2021sharp}.

\newpage
\bibliographystyle{icml2022}
\bibliography{refs}

\newpage
\appendix
\onecolumn

\section{Proof of Structural Results}\label{apx:structural-proofs}

\subsection{Non-overlapping Latent Classes}\label{apx:latent-ind-rep-proofs}

\obsSubadditive*
\begin{proof}
We consider the case of $\beta=1$. The case of general $\beta$ follows immediately.

For $\logloss(v) := \log(1 + \sum_i \exp(-v_i))$, using concavity of $\log$ (Jensen's inequality), and denoting $T := 1 + \sum_{j \notin S} \exp(-v_j)$, we have
\begin{align*}
\textstyle\logloss(v)
~=~ \log\inparen{T + \sum_{i \in S} \exp(-v_i)}
~\ge~ \frac{1}{|S|} \sum_{i \in S} \log\inparen{T + |S| \cdot \exp(-v_i)}
~=~ \frac{1}{|S|}\sum_{i \in S} \logloss(v^{S \gets i})
\end{align*}
For $\hingeloss(v) := \max\set{0, 1 + \max\set{-v_i}}$, using the simple property that $\max\set{a_1, \ldots, a_k} \ge (\sum a_i) / k$, we have
\begin{align*}
\ell(v)
&\textstyle~=~ \max\set{0, 1 + \max_{i\in S}\set{-v_i}, 1 + \max_{j\notin S}\set{-v_j}}\\
&\textstyle~\ge~ \frac{1}{|S|} \sum_{i \in S} \max\set{0, 1 - v_i, 1 + \max_{j \notin S} \set{-v_j}}
~=~ \frac{1}{|S|} \sum_{i \in S} \hingeloss(v^{S \gets i})\qedhere
\end{align*}
\end{proof}

\thmLatentIndep*
\begin{proof}
We show the existence of $\wtilde{f}$ in an existential manner. For a fixed latent class $c^* \in \calC$, sample $x^* \sim \calD_{c^*}$ and define a representation $f_{x^*} : \calX\to \bbS^{d-1}$ as
\begin{align}
	f_{x^*}(x) &~:=~ \begin{cases}
		f(x^*) & \text{if } c(x) = c^*\\
		f(x) & \text{if } c(x) \ne c^*
	\end{cases}
\end{align}
We will show that $\Ex_{x^* \sim \calD_c} \LNCEk(f_{x^*}) \le \LNCEk(f)$. This implies the existence of an $x^*$ in support of $\calD_{c^*}$ such that $\LNCEk(f_{x^*}) \le \LNCEk(f)$.

We can rewrite the NCE loss as
\begin{align}
\LNCEk(f) &~=~ \Ex_{c, c_{1:k} \sim \rho} \ \underbrace{\Ex_{x, x^+ \sim \calD_c} \ \Ex_{\tuple{x_i^- \sim \calD_{c_i}}} \ \ell\inparen{\tuple{f(x)^\top (f(x^+) - f(x_i^-))}_{i=1}^k}}_{=:~\calL_{c,c_{1:k}}(f)}
\end{align}
We will show that for each $c, c_{1:k}$, it holds that $\Ex_{x^* \sim \calD_{c^*}} \calL_{c,c_{1:k}}(f_{x^*}) \le \calL_{c,c_{1:k}}(f)$.

\begin{description}
\item [{\underline{\em \boldmath Case $c = c^*$:}}] For ease of notation let $c_1, \ldots, c_q = c^*$ and $c_{q+1}, \ldots, c_k \ne c^*$. We have
\begin{align}
\calL_{c, c_{1:k}}(f) &~=~ \Ex_{x, x_{q+1:k}^-} \Ex_{x^+, x_{1:q}^-} \ell\inparen{\tuple{f(x)^\top (f(x^+) - f(x_i^-))}_{i=1}^k}\label{eq:c=c*-step-1}\\
&~\ge~ \Ex_{x, x_{q+1:k}^-} \Ex_{x^+, x_{1:q}^-} \ell\inparen{\tuple{f(x)^\top (f(x^+) - f(x_i^-))}_{i=1}^q \circ \tuple{1 - f(x)^\top f(x_i^-))}_{i=q+1}^k}\label{eq:c=c*-step-2}\\
&~\ge~ \Ex_{x, x_{q+1:k}^-} \ell\inparen{\tuple{\Ex_{x^+, x_{1:q}^-} f(x)^\top (f(x^+) - f(x_i^-))}_{i=1}^q \circ \tuple{1 - f(x)^\top f(x_i^-))}_{i=q+1}^k}\label{eq:c=c*-step-3}\\ 
&~=~ \Ex_{x, x_{q+1:k}^-} \ell\inparen{\tuple{0}_{i=1}^q \circ \tuple{1 - f(x)^\top f(x_i^-))}_{i=q+1}^k} \label{eq:c=c*-step-4}
\end{align}
where, \eqref{eq:c=c*-step-2} holds because $f(x)^\top f(x^+) \le 1$ and $\ell(v)$ is non-increasing in each $v_i$, \eqref{eq:c=c*-step-3} holds due to convexity of $\ell$ (Jensen's inequality), and \eqref{eq:c=c*-step-4} holds because $\Ex_{x^+, x_i^- \sim \calD_{c^*}} [f(x^+) - f(x_i^-)] = 0$. Finally, we observe that the quantity in \eqref{eq:c=c*-step-4} is precisely $\Ex_{x^* \sim \calD_{c^*}} \calL_{c, c_{1:k}}(f_{x^*})$, since
\begin{align}
\Ex_{x^* \sim \calD_{c^*}} \calL_{c, c_{1:k}}(f_{x^*}) &~=~ \Ex_{x^* \sim \calD_{c^*}} \Ex_{x, x_{q+1:k}^-} \Ex_{x^+, x_{1:q}^-} \ell\inparen{\tuple{f_{x^*}(x)^\top (f_{x^*}(x^+) - f_{x^*}(x_i^-))}_{i=1}^k}\\
&~=~ \Ex_{x^* \sim \calD_{c^*}} \Ex_{x, x_{q+1:k}^-} \ell\inparen{\tuple{0}_{i=1}^q \circ \tuple{1 - f(x^*)^\top f(x_i^-))}_{i=q+1}^k}
\end{align}
which is same as the quantity in \eqref{eq:c=c*-step-4} up to renaming $x$ by $x^*$ (note: both $x, x^* \sim \calD_{c^*}$ here). For a strictly convex loss, note that our application of Jensen's inequality is tight only when $\Pr[f(x)^\top (f(x^+) - f(x_i^-)) = 0] = 1$, which is the case only if $\Pr_{x, x' \sim \calD_{c^*}}[f(x) = f(x')] = 1$.\info{This part is a bit hand-wavy.}

\item [{\underline{\em \boldmath Case $c \ne c^*$:}}] Again, for ease of notation let $c_1, \ldots, c_q = c^*$ and $c_{q+1}, \ldots, c_k \ne c^*$. We have
\begin{align}
	\calL_{c, c_{1:k}}(f) &~=~ \Ex_{x, x^+ x_{q+1:k}^-} \Ex_{x_{1:q}^- \sim \calD_{c^*}} \ell\inparen{\tuple{f(x)^\top (f(x^+) - f(x_i^-))}_{i=1}^k}\label{eq:cnec*-step-1}\\
	\Ex_{x^* \sim \calD_{c^*}} \calL_{c, c_{1:k}}(f_{x^*}) &~=~ \Ex_{x, x^+ x_{q+1:k}^-} \Ex_{x^* \sim \calD_{c^*}} \ell\inparen{\tuple{f(x)^\top (f(x^+) - f(x^*))}_{i=1}^k}\label{eq:cnec*-step-2}
\end{align}
We will show that for all $x, x^+, x_{q+1:k}^-$, for any loss $\ell$ satisfying \cref{def:sub-additive-loss} it holds that
\[
\Ex_{x_{1:q}^- \sim \calD_{c^*}} \ell\inparen{\tuple{f(x)^\top (f(x^+) - f(x_i^-))}_{i=1}^k} ~\ge~ \Ex_{x^* \sim \calD_{c^*}} \ell\inparen{\tuple{f(x)^\top (f(x^+) - f(x^*))}_{i=1}^k}
\]
For ease of notation, consider a random variable $Z$ that is distributed as $f(x)^{\top} (f(x^+) - f(x^-))$  for $x^- \sim \calD_{c^*}$ (for fixed $x$ and $x^+$).
For any loss $\ell$ satisfying \cref{def:sub-additive-loss}, we have that
\begin{align*}
&\Ex_{Z_{1:q}} \ \ell\inparen{\tuple{Z_i}_{i=1}^q \circ \tuple{f(x)^\top (f(x^+) - f(x_i^-))}_{i=q+1}^k}\\
&~\ge~\Ex_{Z_{1:q}} \insquare{\frac{1}{q} \sum_{j=1}^k \ell\inparen{\tuple{Z_j}_{i=1}^q \circ \tuple{f(x)^\top (f(x^+) - f(x_i^-))}_{i=q+1}^k}}\\
&~=~ \Ex_{Z} \ \ell\inparen{\tuple{Z}_{i=1}^q \circ \tuple{f(x)^\top (f(x^+) - f(x_i^-))}_{i=q+1}^k}
\end{align*}
\end{description}

\noindent Thus, we have $\Ex_{x^* \sim \calD_{c^*}} \LNCEk(f_{x^*}) \le \LNCEk(f)$, thereby completing the proof. The existence of $\wtilde{f}$ follows by iteratively repeating this argument for each latent class $c^* \in \calC$.

Moreover, if $\ell$ is strictly convex loss and $f$ is not latent-indistinguishable, then we have a strict inequality in the case of $c = c^*$ for at least one $c^*$ in this iterative process. Thus, in the case of a strictly convex loss $\ell$, the only NCE optimal representations are almost latent-indistinguishable.
\end{proof}

\subsection{Uniform distribution over classes} \label{apx:proof-equiangular-proof}

\claimEquiangular*
\begin{proof}
We prove the statement for latent-indistinguishable representations $f$. The proof for {\em almost} latent-indistinguishable representations follows similarly.

Let $u_c$ denote the (common) representation for all $x$ in latent class $c$, namely $f(x) = u_{c(x)}$. Observe that $\|\sum_{c \in \calC} u_c\|_2^2 = C + \sum_{c \ne c'} u_c^\top u_{c'} \ge 0$ and hence $\Ex_{c, c' \sim \rho} [u_c^\top u_{c'} \mid c \ne c'] \ge -1/(C-1)$ (under \Cref{ass:uniform} that $\rho$ is uniform over $\calC$). Let $\wtilde{f}$ be an {\em equiangular} representation given as $\wtilde{f}(x) = \wtilde{u}_{c(x)}$ satisfying
\[
\wtilde{u}_c^\top \wtilde{u}_{c'} = \begin{cases} 1 & \text{if } c = c'\\ \Ex_{c, c' \sim \rho}[u_c^\top u_{c'} \mid c \ne c'] & \text{if } c \ne c'\end{cases}\,.
\]
We will show that $\LNCEk(\wtilde{f}) \le \LNCEk(f)$ for any convex loss $\ell$. Observe that since $f$ and $\wtilde{f}$ are latent-indistinguishable, we can write $\LNCEk(f)$ in the following simplified form
\[\LNCEk(f) ~=~ \Ex\limits_{c,c_1, \ldots, c_k \sim \rho} \ \insquare{\ell(1 - u_c^\top u_{c_1}, 1 - u_c^\top u_{c_2}, \ldots, 1 - u_c^\top u_{c_k})}\]

For any permutation $\pi$ over $\calC$, let $f_{\pi}$ denote the representation obtained by permuting the representations of the latent classes, namely, $f_\pi(x) ~:=~ u_{\pi(c(x))}$. We have the following
\begin{align*}
\LNCEk(f)
&~=~ \Ex_\pi \insquare{\LNCEk(f_\pi)} \\
&~=~ \Ex\limits_{c,c_1, \ldots, c_k \sim \rho} \Ex_\pi \ \insquare{\ell(1 - u_{\pi(c)}^\top u_{\pi(c_1)}, 1 - u_{\pi(c)}^\top u_{\pi(c_2)}, \ldots, 1 - u_{\pi(c)}^\top u_{\pi(c_k)})}\\
&~\ge~ \Ex\limits_{c,c_1, \ldots, c_k \sim \rho} \ \insquare{\ell(1 - \Ex_\pi [u_{\pi(c)}^\top u_{\pi(c_1)}], 1 - \Ex[u_{\pi(c)}^\top u_{\pi(c_2)}], \ldots, 1 - \Ex_\pi[u_{\pi(c)}^\top u_{\pi(c_k)}])}\\
&~=~ \Ex\limits_{c,c_1, \ldots, c_k \sim \rho} \ \insquare{\ell(1 - \wtilde{u}_c^\top \wtilde{u}_{c_1}, 1 - \wtilde{u}_c^\top \wtilde{u}_{c_2}, \ldots, 1 - \wtilde{u}_c^\top \wtilde{u}_{c_k})}\\
&~=~ \LNCEk(\wtilde{f})
\end{align*}
where the third step follows from Jensen's inequality, using convexity of $\ell$.

Finally, for any non-increasing loss $\ell$, it is easy to see that among all equiangular representations, the representation minimizing $\LNCEk(\cdot)$ is a Simplex ETF. Moreover, when $\ell$ is strictly convex our application of Jensen's inequality is strict unless $u_c^\top u_{c'}$ is the same for all $c \ne c'$, in other words, the representation is equiangular.
\end{proof}

\Experimental{
\subsection{Optimal Representation for Downstream Supervised Loss}\label{apx:supervised-optimal-proof}

Using similar ideas as in the proof of \cref{thm:latent-ind-rep} and \cref{claim:equiangular}, we show that the Simplex ETF is also the optimal representation minimizing the supervised learning loss.

\thmSupervisedOptimal*

It is easy to see that \cref{thm:supervised-optimal} follows by combining the following two claims.

\begin{claim}\label{claim:supervised-optimal-latent-indistinguishable}
Under \Cref{ass:non-overlap}, for any loss $\ell$, it holds for all representations $f : \calX \to \bbS^{d-1}$, that there exists a latent-indistinguishable representation $\wtilde{f} : \calX \to \bbS^{d-1}$ such that $\Lsup(\wtilde{f}) \le \Lsup(f)$. Moreover, if $\ell$ is strictly convex (e.g. logistic loss), then the inequality above is strict, unless $f$ is almost latent-indistinguishable.
\end{claim}

\begin{claim}\label{claim:supervised-optimal-equiangular}
Under \cref{ass:non-overlap,ass:uniform}, for any convex, non-increasing loss $\ell$, it holds for all (almost) latent-indistinguishable representations $f : \calX \to \bbS^{d-1}$, that $\Lsup(\wtilde{f}) \le \Lsup(f)$ for any Simplex ETF $\wtilde{f} : \calX \to \bbS^{d-1}$.
Moreover if $\ell$ is strictly convex (e.g. logistic loss), then equality holds only if $f$ is an (almost) Simplex ETF representation.
\end{claim}

\begin{proof}[Proof of \cref{claim:supervised-optimal-latent-indistinguishable}]
Given a representation $f : \calX\to\bbS^{d-1}$, let $\set{w_c \mid c \in \calC}$ with $\|w_c\|_2 \le 1$ be such that
\[
\Lsup(f) ~=~ \Ex_{c \sim \rho}~\Ex_{x \sim \calD_c}~\ell\inparen{\tuple{f(x)^\top (w_c - w_{c'}}_{c' \ne c}}\,.
\]
For each $c \in \calC$, fix $u_c \in \argmin_{u} \ell\inparen{\tuple{u^\top (w_c - w_{c'}}_{c' \ne c}}$ and consider the representation $\wtilde{f}:\calX\to \bbS^{d-1}$ given as $\wtilde{f}(x) ~:=~ u_{c(x)}$. By the choice of $\{u_c\}_{c \in \calC}$, it follows that
\[
\Lsup(\wtilde{f}) ~\le~ \Ex_{c \sim \rho}~\Ex_{x \sim \calD_c}~\ell\inparen{\tuple{u_{c(x)}^\top (w_c - w_{c'}}_{c' \ne c}} ~\le~ \Lsup(f)\,.
\]
Moreover, for a strictly convex loss $\ell$, there exist a unique choice of $u_c$ for each $c \in \calC$ and thus the inequality above can be tight only if $f$ is an almost latent-indistinguishable representation.
\end{proof}

\begin{proof}[Proof of \cref{claim:supervised-optimal-equiangular}]
We prove the statement for latent-indistinguishable representations $f$. The proof for {\em almost} latent-indistinguishable representations follows similarly.

Let $u_c$ denote the (common) representation for all $x$ in latent class $c$, namely $f(x) = u_{c(x)}$. Let $\set{w_c \mid c \in \calC}$ with $\|w_c\|_2 \le 1$ be such that $\Lsup(f) = \Ex_{(x, c) \sim \Dsup} \ell\inparen{\tuple{f(x)^\top (w_c - w_{c'}}_{c' \ne c}} = \Ex_{c \sim \rho} \ell\inparen{\tuple{u_c^\top (w_c - w_{c'}}_{c' \ne c}}$. We consider two cases, where $(\sum_{c} u_c)^\top (\sum_c w_c) \ge 0$ and $(\sum_{c} u_c)^\top (\sum_c w_c) < 0$.

In the former case, we have $0 \le (\sum_{c \in \calC} u_c)^\top (\sum_{c \in \calC} w_c) \le \sum_c u_c^\top w_c + \sum_{c \ne c'} u_c^\top w_{c'}$ and hence $\Ex_{c, c' \sim \rho} [u_c^\top w_{c'} \mid c \ne c'] \ge -\Ex_{c \sim \rho} u_c^\top w_c/(C-1)$ (under \Cref{ass:uniform} that $\rho$ is uniform over $\calC$).
Thus, there exist $\set{\wtilde{u}_c \mid c \in\calC}$ and $\set{\wtilde{w}_c \mid c \in \calC}$ such that
\[
\wtilde{u}_c^\top \wtilde{w}_{c'} = \begin{cases} \Ex_{c \sim \rho} u_c^\top w_{c} & \text{if } c = c'\\ \Ex_{c, c' \sim \rho}[u_c^\top w_{c'} \mid c \ne c'] & \text{if } c \ne c'\end{cases},
\]
for example, by taking $\set{u_c \mid c \in \calC}$ to be an equiangular set of unit vectors with $u_c^\top u_{c'} = \Ex_{c, c' \sim \rho}[u_c^\top w_{c'} \mid c \ne c'] / \Ex_{c \sim \rho} u_c^\top w_{c}$ for $c \ne c'$ and taking $w_c := u_c \cdot \Ex_{c \sim \rho} u_c^\top w_{c}$.
Let $\wtilde{f}$ be an {\em equiangular} representation given as $\wtilde{f}(x) = \wtilde{u}_{c(x)}$. Moreover, for any permutation $\pi$ over $\calC$, let $\tuple{u_c^\pi}_{c \in \calC}$ and $\tuple{w_c^\pi}_{c \in \calC}$ denote the tuple obtained by permuting the tuples $\tuple{u_c}_{c\in\calC}$ and $\tuple{w_c}_{c\in\calC}$ with permutation $\pi$. For $\rho$ being the uniform distribution over $\calC$, we have
\begin{align*}
	\Lsup(f)
	&~=~ \Ex_{c \sim \rho} \ell\inparen{\tuple{u_c^\top (w_c - w_{c'})}_{c' \ne c}}\\
	&~=~ \Ex_{\pi} \Ex_{c \sim \rho} \ell\inparen{\tuple{(u_c^\pi)^\top (w_c^\pi - w_{c'}^\pi)}_{c' \ne c}}\\
	&~\ge~ \Ex_{c \sim \rho} \ell\inparen{\tuple{\Ex_{\pi} (u_c^\pi)^\top w_c^\pi - \Ex_{\pi} (u_c^\pi)^\top w_{c'}^\pi}_{c' \ne c}}\\
	&~=~ \Ex_{c \sim \rho} \ell\inparen{\tuple{\wtilde{u}_c^\top (\wtilde{w}_c - \wtilde{w}_{c'}}_{c' \ne c}}\\
	&~\ge~ \Lsup(\wtilde{f})
\end{align*}
where the third step follows from Jensen's inequality, using convexity of $\ell$.

Finally, for any non-increasing loss $\ell$, it is easy to see that among all equiangular $\set{u_c \mid c \in \calC}$ and $w_c \propto u_c$, the optimal choice for minimizing $\Ex_{c \sim \rho} \ell\inparen{\tuple{u_c^\top (w_c - w_{c'}}_{c' \ne c}}$ is where $u_c$'s form a Simplex ETF and $w_c = u_c$. Moreover, when $\ell$ is strictly convex our application of Jensen's inequality is strict unless $u_c^\top u_{c'}$ is the same for all $c \ne c'$, in other words, the representation is equiangular.
\end{proof}
}{} 

\section{More Details about the CIFAR-10/100 Experiments}\label{apx:cifar-experiments}
We describe our experimental setup in full detail. For most of our experiments we train with a ResNet-18 \cite{he2016deep} backbone and a 2-layer projection head affixed on top of it with a ReLU in the middle. The training setup closely follows that of \citet{chen2020simple}. We do not use any weight decay as it might bias us away from seeing a simplex ETF structure. We train using the logistic form of NCE loss and use LARS optimizer \cite{you2017large} with a batch size of $512$, learning rate of $0.2$ which is decayed using a cosine decay after $10$ warmup epochs. Departing from the setting of \citet{chen2020simple}, the final output of the projection head is taken as the representation as this is the vector which is used in computing the loss. This is a $128$ dimensional vector which is normalized to lie inside the unit $\ell_2$ ball.

In addition to the results listed in \cref{sec:cifar-experiments}, we present a few additional observations here. In \cref{tbl:cs-5} we present the cosine similarities matrix for a run with 5 classes. In \cref{fig:scaling_k} we show how the average $\CS$ and $\IntraVar_c$ values scale with the number of negative samples for a fixed batch size.

\begin{figure*}[htb]   
\centering
\minipage{0.4\textwidth}
  \includegraphics[width=\linewidth]{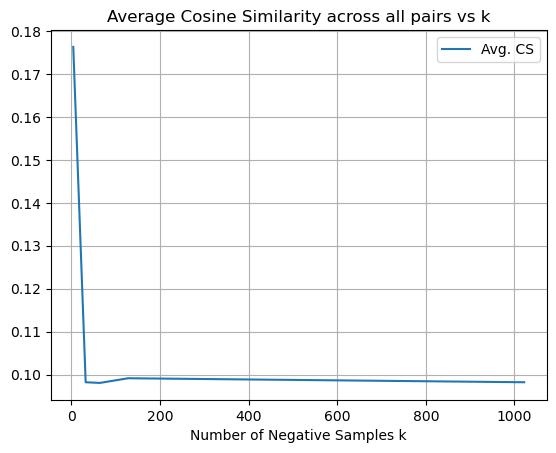}
\endminipage \hspace{5mm}
    \caption{The change in the average cosine similarities is plotted for different values of the number of negative samples $k$. All experiments are done for CIFAR-10 with total number of classes=10.}
    \label{fig:scaling_k}
\end{figure*}

\begin{table}
\begin{center}
\begin{tabular}{ |c|c|c|c|c| } 
 \hline
 1.0 & -0.156 & -0.204 & -0.342 & -0.3413 \\ 
 \hline
 -0.156 & 1.0 & -0.373  & -0.343 & -0.327 \\ 
 \hline
 -0.204 & -0.373  & 1.0 & -0.126 & 0.06 \\ 
 \hline
 -0.342 & -0.343  & -0.126 & 1.0 & -0.173 \\ 
 \hline
 -0.3413 & -0.327 & 0.06 & -0.173 & 1.0 \\ 
 \hline
\end{tabular}
\end{center}
\caption{The cosine similarities between the mean representations of different classes. Shown here for contrastive learning run on a subset of 5 classes from CIFAR-10. Diagonal values are always $1.0$. Our theory expects the off diagonal entries to be $-0.25$}
\label{tbl:cs-5}
\end{table}

\section{Improving Theorem 5 of \citet{ash21investigating}}\label{apx:ash21improvement}
In our work, we showed that in the case of uniform latent classes, the NCE optimal representation doesn't get worse with $k$ in terms of the downstream classification error. However, we could also try to upper bound the $\Lsup(f)$ by some factor $G$ times the $\LNCEk(f)$ for any representation $f$. This is precisely the form of the result obtained by \citet{ash21investigating,saunshi2019theoretical}. In both these works the factor $G$ grows exponentially in $k$. We show that in the case of logistic loss, we can improve this to a factor that is non-increasing and in fact becomes a constant for large enough $k$.

We use the vector $\rho$ to denote the class distribution and $\rho_{\max} = \max \rho_i$ and $\rho_{\min} = \min \rho_i$.
Here we first re-state Theorem 5 of \cite{ash21investigating} using our notation and then proceed to show a stronger result for the setting of $k \ge \rho_{\min}$ when we work with the logistic loss. Recall that in this setting, $1/(1-\rho_{\max})^k$ from their theorem grows exponentially in $k$ suggesting that for large values of $k$, a small $\LNCEk(f)$ may not correspond to a small $\Lsup(f)$. We will show a drastic improvement on this exponential growth with respect to $k$.
\begin{theorem}[Restatement of Theorem 5 of \cite{ash21investigating}]
\label{thm:ash}
    For the logistic loss and any representation $f : \calX \to \bbS^{d-1}$, 
    \begin{align}
        \Lsup(f) \le \frac{2\max\left( 1, \frac{2(1-\rho_{\min})\log C}{k\rho_{\min}}\right)}{(1-\rho_{\max})^k}\left( \LNCEk(f) - \tau_k\Ex_{c,c_i^- \sim \rho^{k+1}} \left[\log\inparen{1+|I|}| I \ne \phi\right]\right)\,, \notag
    \end{align}
    where $I$ is the set of collisions among the $k$ negative samples.
\end{theorem}

\noindent We prove an improved version of \cref{thm:ash} below. Note that we don't have the $\tau_k\Ex_{c,c_i^- \sim \rho^{k+1}} \left[\log\inparen{1+|I|}| I \ne \phi\right]$ term but the coefficient in front is vastly improved from $\frac{2\max\left( 1, \frac{2(1-\rho_{\min})\log C}{k\rho_{\min}}\right)}{(1-\rho_{\max})^k}$ to $4\max\inparen{1, \frac{2(1-\rho_{\min})\log C}{k(1-\rho_{\max})\rho_{\min}}}$.
\begin{theorem}[Improved Theorem 5 of \cite{ash21investigating}]
\label{thm:improved-ash}
    Let $k \ge 1/\rho_{\max}$. For the logistic loss, for any $f : \calX \to \bbS^{d-1}$,
    \info{PK: I replaced $k'$ in place by $k (1 - \rho_{\max})$. The proof defines $k'$ when necessary.}
    \begin{align*}
        \Lsup(f) ~\le~ 4\max\inparen{1, \frac{2(1-\rho_{\min})\log C}{k(1-\rho_{\max}) \rho_{\min}}}\cdot \LNCEk(f)\,.
    \end{align*}
\end{theorem}
\begin{proof}
We recall the sub-addivitity property of logistic loss.
\begin{lemma}[Sub-additivity of Logistic Loss (Lemma 1, \cite{ash21investigating})]
Let $v \in \mathbb{R}^k$ be a vector. For all $I_1, I_2 \subset [k]$, and $S = I_1 \cup I_2$, we have that 
\begin{align}
    \ell(\{v_i\}_{i \in I_1}) ~\le~ \ell(\{v_i\}_{i \in S}) ~\le~ \ell(\{v_i\}_{i \in I_1}) + \ell(\{v_i\}_{i \in I_2})\,. \notag
\end{align}
\end{lemma}
\noindent Following \cite{ash21investigating}, we begin with an application of Jensen's inequality to get
\begin{align}
    \LNCEk(f) &~=~ \Ex_{\DNCE} \insquare{\ell\inparen{\tuple{f(x)^\top (f(x^+) - f(x_i^-))}_{i=1}^k}}\, \notag\\
    &~\ge~ \Ex_{c,c_i^-\sim \rho^{k+1}, x\sim D_c} \insquare{\ell\inparen{\tuple{f(x)^\top (\mu_c - \mu_{c_i^-})}_{i=1}^k}}\,,
\end{align}
where $\mu_c = \Ex_{x \sim D_c}[f(x)]$.
Given $k$ negative samples, if the last $k_1$ of them are collisions, we have from the sub-additivity of the logistic loss,
\begin{align}
    \ell\inparen{\tuple{f(x)^{\top}\inparen{\mu_c - \mu_{c_i}^-} }_{i=1}^k} &~=~  \ell\inparen{\tuple{f(x)^{\top}\inparen{\mu_c - \mu_{c_i}^-}}_{i=1}^{k-k_1} \circ \tuple{0}_{i=k-k_1+1}^k} ~\ge~ \ell\inparen{\tuple{f(x)^{\top}\inparen{\mu_c - \mu_{c_i}^-} }_{i=1}^{k-k_1}}.
\end{align}
Now for any fixed $c$, the probability of a collision for a randomly drawn negative sample is $\rho_c$. Given $k$ negative samples, let $I_c$ denote the set of collisions among the negative samples. For simplicity, we will assume that $k\rho_{\max}$ is an integer. For the most likely class, $|I_{\max}|$ is distributed as $\mathrm{Bin}(k,\rho_{\max})$ and its median is precisely $k\rho_{\max}$. Therefore, we have that
\begin{align}
    \Pr\left[|I_{\max}| > k\rho_{\max} \right] \le 1/2\,.
\end{align}
Let $k' = k(1-\rho_{\max})$.
Now,
\begin{align}
    \LNCEk(f) &~\ge~ \Ex_{c,c_i^-, x\sim D_c} \insquare{\ell\inparen{\tuple{f(x)^\top (\mu_c - \mu_{c_i^-})}_{i=1}^k}} \notag\\
    &~=~ \Ex_{c\sim \rho} \insquare{\Pr\insquare{|I_c| \le k\rho_{\max}} \cdot \Ex_{c_i^-\sim \rho^{k}, x\sim D_c}\insquare{\ell\inparen{\tuple{f(x)^\top (\mu_c - \mu_{c_i^-})}_{i=1}^{k}} \,\Big\vert\, |I_c| \le k\rho_{\max}}} \notag\\
    &~~~~~+ \Ex_{c\sim \rho} \insquare{\Pr\insquare{|I_c| > k\rho_{\max}} \cdot \Ex_{c_i^-\sim \rho^{k}, x\sim D_c}\insquare{\ell\inparen{\tuple{f(x)^\top (\mu_c - \mu_{c_i^-})}_{i=1}^{k}} \,\Big\vert\, |I_c| > k\rho_{\max}}} \notag\\
    &~\ge~ \Ex_{c\sim \rho} \insquare{\Pr\insquare{|I_c| \le k\rho_{\max}} \cdot \Ex_{c_i^-\sim \rho_{-c}^{k'}, x\sim D_c}\insquare{\ell\inparen{\tuple{f(x)^\top (\mu_c - \mu_{c_i^-})}_{i=1}^{k'}}}} \\
    &~\ge~ \Pr\insquare{|I_{\max}| \le k\rho_{\max}} \cdot \Ex_{c\sim \rho,c_i^-\sim \rho_{-c}^{k'}, x\sim D_c} \insquare{\ell\inparen{\tuple{f(x)^\top (\mu_c - \mu_{c_i^-})}_{i=1}^{k'}}}\,.
\end{align}
Next, we use Lemma 4 from \cite{ash21investigating} which gives that for any $c \in \calC$ and any $x$
\begin{align}
    \Ex_{c_i^-\sim \rho_{-c}^{k'}} \insquare{\ell\inparen{\tuple{f(x)^\top (\mu_c - \mu_{c_i^-})}_{i=1}^{k'}}} ~\ge~ \frac{1}{2\left\lceil \frac{2(1-\rho(c))\log C}{(k')\min_{c' \ne c} \rho(c')}\right\rceil}\cdot \ell\inparen{\tuple{f(x)^\top (\mu_c - \mu_{c'})}_{c' \in \mathcal{C}\backslash \{c\}}}.
\end{align}
Substituting these above, we get
\begin{align}
    \LNCEk(f) &\ge \frac{1}{2}\Ex_{c\sim D,c_i^-\sim D_{-c}^{k'}, x\sim D_c} \insquare{\ell\inparen{\tuple{f(x)^\top (\mu_c - \mu_{c_i^-})}_{i=1}^{k'}}} \\
    &\ge \frac{1}{2}\Ex_{c\sim D} \frac{1}{2\left\lceil \frac{2(1-\rho_{\min})\log C}{(k')\rho_{\min}}\right\rceil}\cdot \ell\inparen{\tuple{f(x)^\top (\mu_c - \mu_{c'})}_{c' \in \mathcal{C}\backslash \{c\}}} \\
    &\ge \frac{1}{4\left\lceil \frac{2(1-\rho_{\min})\log C}{(k')\rho_{\min}}\right\rceil}\cdot\Lsup(f, \tuple{\mu_c}_{c \in \calC}) \\
    &\ge \frac{1}{4\left\lceil \frac{2(1-\rho_{\min})\log C}{(k')\rho_{\min}}\right\rceil}\cdot\Lsup(f)
\end{align}
which gives us the claimed result.
\end{proof}


\end{document}